\documentclass{article}

\PassOptionsToPackage{numbers, compress}{natbib}
\usepackage[final]{nips_2017}

\usepackage[utf8]{inputenc} 
\usepackage[T1]{fontenc}    
\usepackage{hyperref}       
\usepackage{url}            
\usepackage{booktabs}       
\usepackage{amsfonts}       
\usepackage{nicefrac}       
\usepackage{microtype}      

\usepackage{xcolor}
\hypersetup{
    colorlinks,
    linkcolor={red!50!black},
    citecolor={blue!50!black},
    urlcolor={blue!80!black}
}

\usepackage{graphicx}
\usepackage{subfigure} 

\usepackage{multirow}

\usepackage{amsmath} 
\usepackage{amsfonts}
\usepackage{amsmath}
\usepackage{amssymb}
\usepackage{amsthm}
\DeclareMathOperator{\diag}{diag}
\DeclareMathOperator{\MSE}{MSE}

\newtheorem{prop}{Proposition}

\title{Tikhonov Regularization for Long Short-Term Memory Networks}

\author{
  Andrei Turkin\\
  National Research University of Electronic Technology,\\
  Federal Research Center "Computer Science and Control" of Russian Academy of Sciences\\
  Moscow, Russia \\
  \texttt{aturkin@org.miet.ru}\\
}

\begin{document}

\maketitle

\begin{abstract}
It is a well-known fact that adding noise to the input data often improves network performance. While the dropout technique may be a cause of memory loss, when it is applied to recurrent connections, Tikhonov regularization, which can be regarded as the training with additive noise, avoids this issue naturally, though it implies regularizer derivation for different architectures. In case of feedforward neural networks this is straightforward, while for networks with recurrent connections and complicated layers it leads to some difficulties. In this paper, a Tikhonov regularizer is derived for Long-Short Term Memory (LSTM) networks. Although it is independent of time for simplicity, it considers interaction between weights of the LSTM unit, which in theory makes it possible to regularize the unit with complicated dependences by using only one parameter that measures the input data perturbation. The regularizer that is proposed in this paper has three parameters: one to control the regularization process, and other two to maintain computation stability while the network is being trained. The theory developed in this paper can be applied to get such regularizers for different recurrent neural networks with Hadamard products and Lipschitz continuous functions.
\end{abstract}

\section{Introduction}
\label{sec:intro}
A recurrent neural network with the many-to-one architecture can be viewed as a mapping $M: \mathbb{R}^X \times T \rightarrow \mathbb{R}^Y$ with a set of parameters $\theta$, where $T \subset \mathbb{N}$ is the set of indexes each of which is regarded as time when some $x(t) \in \mathbb{R}^X$, $t \in T$ was taken. In this formulation, the input data are the following set of inputs: $I(x(t),l)=\{ x(s) \in \mathbb{R}^{X}, s\in \overline{t-l,t} \}$, the output data are $y(t) \in \mathbb{R}^Y$. 

One way to construct the mapping is to use LSTM units. The concept was introduced in \cite{Hochreiter1997LongMemory} as a remedy to vanishing gradient problem, refined in \cite{Gers2000a} and in later papers (see, for instance, \cite{Graves2005FramewiseArchitectures, Jozefowicz2015AnArchitectures}). The LSTM unit has three gates: input, output, and forget ones that are used to control the data flow through the unit. Its input and the gates have parameters that must be trained with some regularization, which often improves network performance and prevents overfitting.
Despite the tremendous performance gain for many applications and abundance of techniques to regularize networks, including dropout \cite{Hinton2012, Srivastava2014}, weight decay ($L_2$ regularization), the standard regularization approach, Recurrent Neural Networks in general -- and LSTMs in particular -- may suffer from overfitting. The usage the techniques for feedforward neural networks is straightforward, though their application to RNNs leads to some difficulties. First, when dropout is applied to recurrent connections, it may cause the memory loss problem that the authors of \cite{Gal2016, Semeniuta2016, Zaremba2014} tried to avoid. Second, though $L_2$ regularization can be used, it is not obvious how it must be applied: whether one regularization parameter should be used or several ones to regularize differently the parts of the unit. Note the latter case is computationally intense than the former one, which leads to slower training, since it is necessary to get the optimal values.

It is feasible to address these problems by derivation a regularizer for LSTM unit that is based on the Tikhonov regularization technique. In \cite{Bishop1995TrainingRegularization} it was shown that adding noise to initial data is equivalent to Tikhonov regularization. Almost at the time the authors of \cite{Wu1996ANetworks} showed a possibility to apply the concept to recurrent neural networks, as the regularizer can be obtained by calculating the upper bound of the squared output disturbance $\|M(I(\hat{x}(t),l);\theta)-M(I(x(t),l);\theta)\|^2$, where $\hat{x}(t)=x(t)+\epsilon$ and $\epsilon$ is the independent random noise with zero mean and variance $\sigma_\epsilon$.

In this paper, the upper bound is calculated to get a regularizer for LSTM networks in case of solving a regression task with the sum-of-squares objective, though the regularizer can be derived for any other loss.

The paper is organized as follows. Section \ref{sec:netreg} describes the network architecture and the regularizer, which is derived by assessing the upper bound of the output perturbation. Section \ref{sec:LSTMreg} provides the theoretical justification for the form of the LSTM regularizer. Section \ref{sec:netlearning} describes the learning procedure with the regularizer derived previously and the relaxed optimization problem. Section \ref{sec:conclus} concludes the paper.

\section{The Output Perturbation}
\label{sec:netreg}
It is assumed that a layered network topology with the LSTM units is used. Since it is not important for further analysis which output is used, the standard dense layer with the sigmoid function is considered:
\begin{equation}
\label{eq:outputlayer}
M(I(x(t),l);\theta)=y(t)=\sigma(W_{hy}h(t))
\end{equation}
The objective is to assess the upper bound of the output perturbation $\sigma_y^2(t)=\| \hat{y}(t)-y(t) \|^2$, which is the result of the input perturbation $\| \hat{x}(t) - x(t) \|^2$. Thus, it is possible to write the following equation for output perturbation:
\begin{equation}
\label{eq:outputperturb}
\sigma_y^2(t) = \| \sigma(W_{hy}\hat{h}(t))-\sigma(W_{hy}h(t))\|^2 
\end{equation}
Obviously, the upper bound for \eqref{eq:outputperturb} can be found by applying the mean value theorem so the result can be written as follows:
\begin{equation}
\label{eq:OututPerturbUB}
\sigma_y^2(t) \leq  \|\diag \left( \sigma'(\xi_y) \right) \|^2 \| W_{hy}\|^2 \| \hat{h}(t) - h(t)\|^2
\end{equation}
where $\xi_y = [\xi_y^1,...,\xi_y^n]^T$ denotes a point, which is somewhere in between $W_{hy}\hat{h}(t)$ and $W_{hy}h(t)$.

Considering that $\sigma'(x)=(1-\sigma(x))\sigma(x)$ and $\alpha = \max_x(\sigma'(x))=1/4$ for any point $x$, it is possible to write the following equation.
\begin{equation}
\sigma_y^2(t) \leq \alpha^2 \| W_{hy}\|^2 \|\hat{h}(t) - h(t)\|^2.
\end{equation}
The output perturbation depends on the LSTM layer perturbation $\sigma_h^2=\|\hat{h}(t) - h(t)\|^2$ and on the dense layer parameters only. Therefore, the upper bound of it must be assessed to get the regularizer for the network.

\section{The LSTM unit output perturbation}
\label{sec:LSTMreg}
The LSTM unit \cite{Gers2000a} can be described by using the following equations:
\begin{align}
h(t)&=\tanh(s(t)) \odot o(t)=co(t) \odot o(t), \label{eq:h}\\
co(t)&=\tanh(s(t)), \label{eq:co}\\
o(t)&=\sigma (W_{ox} x(t)+W_{oh} h(t-1)+b_{o})=\sigma(net^{o}(t)), \label{eq:o}\\
s(t)&=s(t-1) \odot f(t) + i(t) \odot ci(t), \label{eq:s}\\
f(t)&=\sigma (W_{fx} x(t)+W_{fh} h(t-1)+b_{f})=\sigma(net^f(t)), \label{eq:f}\\
i(t)&=\sigma (W_{ix} x(t)+W_{ih} h(t-1)+b_{i})=\sigma(net^i(t)), \label{eq:i}\\
ci(t)&=\tanh (W_{cix} x(t)+W_{cih} h(t-1)+b_{ci})=\tanh (net^{ci}(t)), \label{eq:ci}
\end{align}
where $W_{*}$ are the weight matrices and $b_{*}$ are the biases, $h(s) \in \mathbb{R}^{N_h}$,  $s=1,..,t$, $net^{u}=W_{ux} x(t)+W_{uh} h(t-1)+b_{u}$, $u \in \{i,o,f,ci\}$.

Considering the equations \eqref{eq:outputlayer} and \eqref{eq:h}-\eqref{eq:ci}, it can be concluded that the objective model $M(I(x(t),l);\theta)$ has the following set of parameters: 
\begin{equation}
\theta=\{W_{hy}\} \cup \{W_{uv}:u \in \{i,o,f,ci\},v \in \{x,h\}\} \cup \{b_{u}:u \in \{i,o,f,ci\}\}.
\end{equation}

\subsection{The upper bound of the recurrent connection perturbation}
Before finding the upper bound of $\sigma^2_h$, it is necessary to prove the following 
\begin{prop}
\label{prop:difference}
Suppose that $U$ is an open set in $\mathbb{R}^n$, $[a_i, b_i],[\hat{a}_i,\hat{b}_i]\in U, i=1,...,n$ such that $U$ contains the line segment $L$ from $[a_i, b_i]$ to $[\hat{a}_i,\hat{b}_i]$ and $f_1$ and $f_2$ are differentiable real-valued function on $U$, then the upper bound of the difference of Hadamard products can be found as follows
\begin{equation}
\label{eq:normhdiff}
\| f_1(\hat{a}) \odot f_2(\hat{b}) - f_1(a) \odot f_2(b) \|^2 \leq \| \diag (\|\nabla F(c)\|^2)\|^2 (\sigma_a^2 + \sigma_b^2),
\end{equation}
where $F(x,y)=f_1(x)f_2(y)$ for some points $x,y\in\mathbb{R}$, $c\in L$, and perturbations $\sigma_a = \|\hat{a}_i - a_i \|$ and $\sigma_b = \|\hat{b}_i - b_i \|$, $c=[c_1,...,c_n]^T$, $c_i=(1-u)[a_i, b_i]^T + u[\hat{a}_i, \hat{b}_i]^T$.
\end{prop}

\begin{proof}
Applying the mean value theorem to an $i$th component of the vector from the left side of the Equation \eqref{eq:normhdiff}, it is possible to write that
\begin{equation*}
\begin{split}
f_1(\hat{a}_i) f_2(\hat{b}_i) - f_1(a_i) f_2(b_i)=\nabla F(c_i)\cdot ([\hat{a}_i, \hat{b}_i]^T - [a_i, b_i]^T).
\end{split}
\end{equation*}

Based on this fact, the norm of the difference of Hadamard products for two functions $f_1$ and $f_2$ can be rewritten as follows: 
\begin{equation}
\| f_1(\hat{a}) \odot f_2(\hat{b}) - f_1(a) \odot f_2(b) \|^2 = \sum_{i=1}^n (f_1(\hat{a}_i) f_2(\hat{b}_i) - f_1(a_i) f_2(b_i))^2
\end{equation}
Therefore,
\begin{equation}
\| f_1(\hat{a}) \odot f_2(\hat{b}) - f_1(a) \odot f_2(b) \|^2 = \sum_{i=1}^n (\nabla F(c_i)\cdot([\hat{a}_i, \hat{b}_i]^T - [a_i, b_i]^T))^2,
\end{equation}

Applying the Cauchy inequality, one can get
\begin{equation}
\begin{split}
&\| f_1(\hat{a}) \odot f_2(\hat{b}) - f_1(a) \odot f_2(b) \|^2 \leq \sum_{i=1}^n \|\nabla F(c_i)\|^2\|[\hat{a}_i, \hat{b}_i]^T - [a_i, b_i]^T\|^2=\\
&=\sum_{i=1}^n \|\nabla F(c_i)\|^2(\hat{a}_i - a_i)^2+\|\nabla F(c_i)\|^2(\hat{b}_i - b_i)^2 = \\
&=\|\diag(\|\nabla F(c)\|)(\hat{a} - a)\|^2+\|\diag(\|\nabla F(c)\|)(\hat{b} - b)\|^2,
\end{split}
\end{equation}

Applying the Cauchy inequality again to the previously obtained equation, it is possible to get the desired result.
\end{proof}

Considering the equations \eqref{eq:h}, \eqref{eq:co}, \eqref{eq:o}, and the Proposition \eqref{prop:difference}, it is possible to write the equation for $\sigma_h^2(t)=\|\hat{h}(t) - h(t) \|^2$ as follows
\begin{equation}
\label{eq:h_upperbound}
\sigma_h^2(t) \leq \beta^2 (\sigma^2_s(t) + \sigma^2_{net^o}(t)),
\end{equation}
where $\beta$ is assessed as
\begin{equation}
\beta = \max_{\xi}(\| \diag (\|\nabla G(\xi)\|)\|^2)=17/16 \footnote{Since $G'_{\xi_1}(\xi)=\tanh'(\xi_1)\sigma(\xi_2)$ and $G'_{\xi_2}(\xi)=\tanh(\xi_1)\sigma'(\xi_2)$.}, 
\end{equation}
and $G = \tanh(\xi_1)\sigma(\xi_2)$, $\xi=[\xi_, \xi_2]^T \in L$, $L$ is the line segment from $[s(t),net^o(t)]^T$ to $[\hat{s}(t), \hat{net}^o(t)]^T$.

Considering the equation \eqref{eq:h_upperbound}, it is possible to state that in order to minimize $\sigma_h^2(t)$, it is necessary to assess the following two perturbations:
\begin{equation}
\label{eq:outdisturb}
\sigma_{net^o}^2(t) = \| \hat{net}^o(t)-net^o(t) \|^2
\end{equation}
and
\begin{equation}
\label{eq:memdisturb}
\sigma_s^2(t) = \| \hat{s}(t)-s(t) \|^2.
\end{equation}

\subsection{The upper bound of the output gate perturbation}

The upper bound of the output gate perturbation can be assessed by using the following

\begin{prop}
\label{prop:OutNetDist}
It holds that
\begin{equation*} 
\sigma_{net^o}(t) \leq \frac{\beta \|W_{oh}\| \sigma_s(t) + \|W_{ox}\| \sigma_x(t)}{\exp\bigl(1-\beta \|W_{oh}\|\bigr)},
\end{equation*}
where $\beta = \| \diag (\nabla G(\xi)) \|$, $\xi \in L$, $L$ is the line segment from $[s(t),net^o(t)]^T$ to $[\hat{s}(t), \hat{net}^o(t)]^T$.
\end{prop}

\begin{proof}
Let $\delta_x(t)$ and $\delta_h(t)$ be the following differences: $\delta_{x}(t)=\hat{x}(t) - x(t)$, $\delta_h(t) = \hat{h}(t) - h(t)$. Then applying the equation \eqref{eq:o} to \eqref{eq:outdisturb}, we can get the following result:
\begin{equation}
\begin{split}
\sigma_{net^o}^2(t) &=\| W_{ox} \delta_x + W_{oh} \delta_h(t-1) \|^2=\\
&\| W_{ox} \delta_x(t) + W_{oh} (\hat{co}(t-1) \odot\hat{o}(t-1) - co(t-1) \odot o(t))) \|^2=\\
&\| W_{ox} \delta_x(t) + W_{oh} (\hat{co}(t-1)\odot\sigma(\hat{net}^o(t-1))-co(t-1) \odot \sigma (net^o(t-1))) \|^2
\end{split}
\end{equation}

Considering the equation \eqref{eq:o} and the following fact
\begin{equation}
\begin{split}
\frac{dnet^o(t)}{dt}=\lim_{\tau \rightarrow 0}\frac{net^o(t+\tau)-net^o(t)}{\tau};
\end{split}
\end{equation}

one can write the following dynamic function for some $\tau \rightarrow 0$ \cite{Wu1996ANetworks}:
\begin{equation}
\frac{net^o(t)}{dt}=\frac{W_{oh}h(t)-net^o(t)+W_{ox}x(t')+b_o}{\tau},
\end{equation}
where $t' = t+\tau$.

If $\sigma_{net^o}^2(t)=\| \hat{net}^o(t) - net^o(t) \|^2$ then its derivative can be estimated as follows
\begin{equation}
\label{eq:sigma2net'}
\frac{\sigma_{net^o}^2(t)}{dt}=2\left[ \hat{net}(t) - net(t) \right]^T \cdot \left[ \frac{\hat{net}^o(t)}{dt} - \frac{net^o(t)}{dt} \right]
\end{equation}

Therefore, considering that $\delta_{net^o}(t) = \hat{net}^o(t)-net^o(t)$ and $\tau > 0$, it is possible to get the following equation for its derivative: 
\begin{equation}
\label{eq:deltanet'}
\begin{split}
&\frac{\delta_{net^o}(t)}{dt} = \lim_{\tau \rightarrow 0} \frac{ W_{oh}\delta_h(t) - \delta_{net^o}(t) + W_{ox}\delta_x(t')}{\tau} ,
\end{split}
\end{equation}
where $\delta_{net^o}(t') = \hat{net}(t') - net(t')$. 

Thus, applying the equation \eqref{eq:deltanet'} to \eqref{eq:sigma2net'}, the latter one can be rewritten as
\begin{equation}
\begin{split}
\frac{d\sigma_{net^o}^2(t)}{dt}=\lim_{\tau \rightarrow 0}\frac{2(A+B-C)}{\tau},
\end{split}
\end{equation}
where $A=[ \delta_{net^o}(t) ]^T [W_{oh}\delta_h(t)]$, $B= [ \delta_{net^o}(t) ]^T [W_{ox}\delta_{x}(t')]$, $C=\sigma_{net^o}^2(t)$.

The upper bound of \eqref{eq:sigma2net'} can be found in the following three steps. First, consider the equation \eqref{eq:h},  the proposition \ref{prop:difference}, and the H\"{o}lder's Inequality \footnote{Since $(\sqrt{a_1}+\sqrt{a_2})^2=a_1^2+a^2+2\sqrt{a_1a_2}\geq a_1+a_2$, then $\sqrt{a_1+a_2}\leq \sqrt{a_1}+\sqrt{a_2}$}, then the upper bound of $A$ can be obtained as

\begin{equation}
\|A\| \leq \beta \|W_{oh}\|\sigma_{net^o}(t)(\sigma_{s}(t) + \sigma_{net^o}(t) ). 
\end{equation}

Second, the upper bound of $B$ is
\begin{equation}
\|B \| \leq \|W_{ox}\| \sigma_{net^o}(t) \sigma_x(t')
\end{equation}

Thus, after some simplifications the upper bound of \eqref{eq:sigma2net'} can be written as
\begin{equation}
\frac{d\sigma_{net^o}^2(t)}{dt} \leq 2\sigma_{net^o}(t) (a \sigma_{net^o}(t)+ b \sigma_s(t) + \sigma_x(t')),
\end{equation}
where $a=\frac{1}{\tau} \left( \|W_{oh}\| \beta - 1 \right)$, $b=\frac{1}{\tau} \|W_{oh}\| \beta$, and $c=\frac{1}{\tau} \|W_{ox}\| $ for some $\tau \rightarrow 0$.

Therefore,
\begin{equation} 
\label{eq:netode1}
\frac{d\sigma_{net^o}(t)}{dt} \leq a \sigma_{net^o}(t) + b\sigma_s(t) +c\sigma_x(t')
\end{equation}

Let us assume that the input perturbation $\sigma_x(t)$ and the memory perturbation $\sigma_s(t)$ are either constants or change more slowly than $\sigma_{net^o}(t)$. Thus, the equation \eqref{eq:netode1} can be rewritten as follows
\begin{equation} \label{eq:netode2}
\frac{d\sigma_{net^o}(t)}{dt} \leq a  \sigma_{net^o}(t) + d,
\end{equation}
where $d=b\sigma_s(t) + c\sigma_x(t)$

Applying the Gr{\"o}nwall inequality (see, for instance, \citep{Pachpatte1998}) to the equation \eqref{eq:netode2} for $t \in [t-\tau, t]$, one can get the upper bound of $\sigma_{net^o}(t)$:
\begin{equation} \label{eq:netsol}
\sigma_{net^o}(t) \leq \tau(b\sigma_s(t) +c\sigma_x(t))\exp(\tau a)
\end{equation}

Substituting the previously defined constants, we can end up with the upper bound of $\sigma_{net^o}(t)$:
\begin{equation} 
\label{eq:netsoltau1}
\sigma_{net^o}(t) \leq \frac{\beta \|W_{oh}\| \sigma_s(t) + \|W_{ox}\| \sigma_x(t)}{\exp\bigl(1-\beta \|W_{oh}\|\bigr)},
\end{equation}
which proves the proposition.
\end{proof}

It should be noted that for the purpose of computational stability it is assumed that $\beta \|W_{oh}\| \leq 1$.

\subsection{The upper bound of the memory perturbation}
In order to minimize the output gate perturbation, it is necessary to take into account the memory perturbation \eqref{eq:memdisturb}. This perturbation is the result of applying the following functions to the input data of the unit: the forget gate function ($f$), the input gate function ($i$), and the input of the unit function ($ci$). Thus, it is possible to write the following equation for the memory perturbation:
\begin{align}
&\sigma_s^2(t)=\| \hat{s}(t) - s(t) \|^2\\
&\hat{s}(t) = s(t-1)\odot \hat{f}(t) + \hat{i}(t)\odot \hat{ci}(t)\\
&s(t) = s(t-1)\odot f(t) + i(t)\odot ci(t)
\end{align}
Therefore, the memory perturbation can be assessed by finding the upper bound of the following norm
\begin{equation}
\label{eq:sigma2s}
\sigma_s^2(t)= \| s(t-1) \odot \delta_f(t) + \delta_{i,ci}(t) \|^2,
\end{equation}
where $\delta_{i,ci}(t) = \hat{i}(t)\odot \hat{ci}(t) - i(t)\odot ci(t)$.

Considering the equation \eqref{eq:sigma2s}, it is possible to write the following
\begin{prop}
\label{prop:MemDist}
It holds that
\begin{equation}
\sigma_s(t) \leq \gamma_{x}\sigma_x(t) + \gamma_{h}\sigma_h(t),
\end{equation}
where $\overline{\gamma}_{x}=\alpha \|W_{fx}\| + \beta( \|W_{ix}\| + \|W_{cix}\|)$, $\overline{\gamma}_{h}=\alpha \|W_{fh}\| + \beta(\|W_{ih}\| + \|W_{cih}\|)$.
\end{prop}

\begin{proof}
Like in the proposition \ref{prop:OutNetDist}, it is possible to write the following equation for the derivative for $\sigma^2_s(t)$:
\begin{equation}
\label{eq:dsigma2dt}
\frac{d\sigma^2_s(t)}{dt}=2\left[\hat{s}(t)-s(t)\right]^T\left[\frac{d\hat{s}(t)}{dt}-\frac{ds(t)}{dt}\right]
\end{equation}
First, it is necessary to rewrite the second part of the equation by using the following one:
\begin{equation}
\label{eq:dshatdt}
\frac{d\hat{s}(t)}{dt}=\lim_{\tau\rightarrow0} \frac{s(t) \odot \hat{f}(t') + \hat{i}(t') \odot \hat{ci}(t') - \hat{s}(t)}{\tau} 
\end{equation}
\begin{equation}
\label{eq:dsdt}
\frac{ds(t)}{dt}=\lim_{\tau\rightarrow0} \frac{s(t) \odot f(t')+ i(t') \odot ci(t') - s(t)}{\tau} 
\end{equation}
Therefore, the second part of \eqref{eq:dsigma2dt} can be rewritten by using the equations \eqref{eq:dshatdt} and \eqref{eq:dsdt} as follows:
\begin{equation}
\label{eq:ddeltasdt}
\frac{d\delta_s(t)}{dt}=\frac{1}{\tau}\left( s(t)\odot \delta_f(t')+\delta_{i,ci}(t')-\delta_s(t) \right),
\end{equation}
where $\delta_f(t') = \hat{f}(t') - f(t')$, $\delta_s(t)=\hat{s}(t) - s(t)$, $\delta_{i,ci}(t')=\hat{i}(t') \odot \hat{ci}(t')-i(t') \odot ci(t')$.

Considering the equation \eqref{eq:ddeltasdt}, the equation \eqref{eq:dsigma2dt} can be rewritten as follows:
\begin{equation}
\frac{d\sigma^2_s(t)}{dt}=\frac{2}{\tau}\delta_s(t)^T\left( s(t)\odot \delta_f(t')+\delta_{i,ci}(t')\right)-\frac{2}{\tau}\delta^2_s(t) 
\end{equation}
Due to the fact that $s(t) \in (-1;1)$, it is possible to apply Lemma 2 from \cite{Pachpatte1996} to get the following equation:
\begin{equation}
\frac{d\sigma^2_s(t)}{dt}\leq\frac{2}{\tau}\delta_s(t)^T\left(\delta_f(t')+\delta_{i,ci}(t')\right)-\frac{2}{\tau}\delta^2_s(t) 
\end{equation}
Applying the mean value theorem, we can find an upper bound of the first part of the equation as follows:
\begin{equation}
\| \delta_f(t') \| \leq \gamma_f \left(\|W_{fx}\|\sigma_x(t')+\|W_{fh}\|\sigma_h(t)\right), 
\end{equation}
where $\sigma_h(t)$ is the recurrent connection perturbation, $\gamma_f = \max_\xi \|\diag(\sigma'(\xi))\|$, $\xi = [net_1^*(t),...,net_n^*(t)]^T$ denotes a point, which is between $\min(\hat{net}^f_j(t),net^f_j(t))$ and $\max(\hat{net}^f_j(t),net^f_j(t))$, $j=1,...,n$; therefore $\gamma_f = \alpha$.

Applying Proposition \ref{prop:difference}, the upper bound of the second part squared can be assessed as
\begin{equation}
\| \delta_{i,ci}(t')\|^2 \leq \gamma_{i,ci}^2 (\sigma_{net^i}^2(t') + \sigma_{net^{ci}}^2(t')),
\end{equation}
where $\gamma_{i,ci}\leq\max_{\xi}(\|\nabla G(\xi)\|)\|^2)$, $G = \tanh(\xi_1)\sigma(\xi_2)$, $\xi=[\xi_, \xi_2]^T$; therefore, $\gamma_{i,ci} = \beta$. 

By applying a previously used inequality ($\sqrt{a_1+a_2}\leq \sqrt{a_1}+\sqrt{a_2} $), it is possible to write that
\begin{equation}
\| \delta_{i,ci}(t')\| \leq \gamma_{i,ci}( \sigma_x(t')(\|W_{ix}\|+\|W_{cix}\|) +\sigma_h(t)(\|W_{ih}\|+ \|W_{cih}\|))
\end{equation}
Therefore, the upper bound of \eqref{eq:dsigma2dt} is 
\begin{equation}
\frac{d\sigma^2_s(t)}{dt} \leq \frac{2}{\tau}\left(\sigma_s(t)\left( \gamma_x \sigma_x(t') + \gamma_h\sigma_h(t)\right) -\sigma_s^2(t)\right),
\end{equation}
where $\gamma_{x}=\gamma_f \|W_{fx}\| + \gamma_{i,ci}( \|W_{ix}\| + \|W_{cix}\|)$, $\gamma_{h}=\gamma_f \|W_{fh}\| + \gamma_{i,ci}(\|W_{ih}\| + \|W_{cih}\|)$

Applying Lemma 2 from \cite{Pachpatte1996}, the upper bound of $\frac{d\sigma_s(t)}{dt}$ can be calculated as follows
\begin{equation}
\label{eq:ubssigmadt}
\frac{d\sigma_s(t)}{dt} \leq \frac{1}{\tau}\left( \gamma_s - \sigma_s(t)\right)\leq \frac{1}{\tau}\gamma_s,
\end{equation}
where $\gamma_s=\gamma_{x}\sigma_x(t') + \gamma_{h}\sigma_h(t)$.

Assuming that $\sigma_x(t)$ and $\sigma_h(t)$ are either constants or change more slowly than $\sigma_{s}(t)$, one can get the following result for the interval $[t-\tau, t]$:
\begin{equation}
\sigma_s(t)\leq \gamma_s,
\end{equation}
which proves the proposition
\end{proof}

\subsection{The upper bound of the output perturbation}

Based on Proposition \ref{prop:MemDist}, it is possible to rewrite equation \eqref{eq:h_upperbound} as follows:
\begin{equation}
\sigma_h^2(t) \leq \rho_x^2(\theta)\sigma_x^2(t)+\rho_h^2(\theta)\sigma_h^2(t)
\end{equation}
where $\rho_x$ and $\rho_h$ are independent of the time variable $t$ and can be calculated based on the parameters of the model only:
\begin{equation}
\rho_x(\theta)=\sqrt{2}\biggl(\overline{\gamma}_x+\frac{\sqrt{2}(\beta\overline{\gamma}_x\|W_{oh}\| + \|W_{ox}\|)}{\exp(1-\beta \|W_{oh}\|)}\bigr),
\end{equation}
\begin{equation}
\rho_h(\theta)=\sqrt{2}\biggl(\overline{\gamma}_h+\frac{\beta \overline{\gamma}_h \|W_{oh}\|}{\exp(1-\beta \|W_{oh}\|)}\biggr),
\end{equation}

After some simplifications, it is possible to conclude that 
\begin{equation}
\label{eq:h_resupperbound}
\begin{split}
&\sigma_h^2(t) \leq \frac{\rho_x^2(\theta)}{1-\rho_h^2(\theta)}\sigma_x^2(t),\\
&\textit{s.t. } \rho_h^2(\theta) < 1; \space \beta \|W_{oh}\| \leq 1.
\end{split}
\end{equation}

\section{The Learning Procedure}
\label{sec:netlearning}
Considering \eqref{eq:h_resupperbound} and the upper bound of the output perturbation:
\begin{equation}
\label{eq:y_resupperbound}
\sigma_y^2(t) \leq \|W_{hy}\|^2 \frac{\rho_x^2(\theta)}{1-\rho_h^2(\theta)}\sigma_x^2(t),
\end{equation}
the regularizer $R(\theta)$ can be written as follows
\begin{equation}
\label{eq:ResObjFunc}
\begin{split}
&R(\theta) = \lambda_{S} \|W_{hy}\|^2 \frac{\rho_x^2(\theta)}{1-\rho_h^2(\theta)},\\
&\textit{s.t. } \rho_h^2(\theta) < 1; \space \gamma \|W_{oh}\| \leq 1.
\end{split}
\end{equation}
where $\lambda_{S}$ is a constant that measures the degree of the input perturbation.

It is possible to relax this problem to get the following objective function:
\begin{equation}
\label{eq:RelResObjFunc}
L(\theta) = \MSE(\theta) + R(\theta) + \lambda_1(\rho_h^2(\theta)-1)_+ + \lambda_2(\gamma \|W_{oh}\|-1)_+,
\end{equation}
where $(u)_+=\max(0,u)$ and $\lambda_{S}$, $\lambda_1$, and $\lambda_2$ are the parameters that must be assessed during the training procedure based on the model evaluation criterion.

The complex regularizer that is the right part of \eqref{eq:RelResObjFunc} has three parameters: $\lambda_S = \lambda_S(\sigma_\epsilon)$, which is the main parameter of the regularization, and $\lambda_1$, $\lambda_2$, which are used to maintain computation stability during the training.

\section{Conclusion}
\label{sec:conclus}
In this paper, the Tikhonov regularizer is derived for the LSTM unit by finding the upper bound of the output perturbation, which is the difference between the actual output of the network and the one that is observed if the noise is added to the inputs of the network. The regularizer has three parameters: the first one measures the degree of input perturbation, thus it controls the regularization process, the other two are used to maintain computation stability of the regularization. The regularizer can be used to approach the overfitting problem in LSTM networks by taking into account not only the weights of the gates independently, but also the interaction between them as parts of the LSTM complex structure. The mathematical justification of the proposed regularization derivation is provided, which enables to get regularizers for different architectures.

\bibliographystyle{plain}
\bibliography{arXiv2017}

\end{document}